\documentclass{article} 

\usepackage{icml2014}

\usepackage[pdftex]{graphicx}
\usepackage{natbib}
\usepackage{times}
\usepackage{amsthm}
\usepackage{amsmath}
\usepackage{amsfonts}
\usepackage{graphicx}
\usepackage{epstopdf}
\usepackage{algorithm}
\usepackage{subfigure}
\usepackage{algorithmic}
\usepackage{hyperref}

\newcommand {\gtrsim} {\ {\raise-.5ex\hbox{$\buildrel>\over\sim$}}\ }
\newcommand {\lesim} {\ {\raise-.5ex\hbox{$\buildrel<\over\sim$}}\ }

\icmltitlerunning{LDA Spectral Model Selction}

\begin{document}

\twocolumn[
\icmltitle{Fast, Guaranteed Spectral Model Selection for Topic Models}

\icmlauthor{E.D. Guti\'errez}{edg@icsi.berkeley.edu}
\icmladdress{UCSD Department of Cognitive Science,
            9500 Gilman Drive, La Jolla, CA 92093 USA}
			
\icmlkeywords{LDA, mixture models, model order, spectral learning, method of moments}

\vskip 0.3in
]
\newtheorem{thm}{Theorem}[section]
\newtheorem{lemma}[thm]{Lemma}
\newtheorem{proposition}[thm]{Proposition}
\newtheorem{corollary}[thm]{Corollary}
\newtheorem{fact}[thm]{Fact}
\newtheorem{definition}[thm]{Definition}
\newtheorem{remark}[thm]{Remark}

\begin{abstract}
The question of how to determine the number of independent latent factors, or topics, in Latent Dirichlet Allocation (LDA) is of great practical importance.  In most applications, the exact number of topics is unknown, and depends on the application and the size of the data set.  We introduce a spectral model selection procedure for topic number estimation that does not require learning the model's latent parameters beforehand and comes with probabilistic guarantees.   The procedure is motivated by the spectral learning approach and relies on adaptations of results from random matrix theory.  In a simulation experiment taken from the nonparametric Bayesian literature, this procedure outperforms the nonparametric Bayesian approach in both accuracy and speed.  We also discuss some implications of our results for the sample complexity and accuracy of popular spectral learning algorithms for LDA. The principles underlying the procedure can be extended to spectral learning algorithms for other exchangeable mixture models with similar conditional independence properties.
\end{abstract}

\section{Introduction}
The question of how to determine the model order--that is, the number of independent latent factors--in mixture models such as Latent Dirichlet Allocation \cite{blei_ng_jordan_03} is of great practical importance.  These models are widely used for tasks ranging from bioinformatics to computer vision to natural language processing.  Finding the least number of latent factors that explains the data improves predictive performance, as well as increasing computational and storage efficiency.  In most appplications, the exact number of latent factors (also known as topics or components) is unknown: model order often depends on the application and increases as the data set grows.  For a fixed training set, the user can subjectively fine-tune the number of topics or optimize it according to objective measures of fit along with the other parameters of  the model, but this is a time-consuming process, and it is not intuitively clear how to increase the number of topics as new data points are encountered without an additional round of fine-tuning.  Moreover, spectral learning procedures can fail if the number of latent factors is underestimated.
\\
In this paper, we present a simple and efficient procedure that estimates model order from the spectral characteristics of the sample cross-correlation matrix of the observed data.  We focus on LDA in this paper in order to illustrate our approach, but our principles be extended to other mixture models with similar conditional independence properties.  Unlike previous approaches to model order selection, the resulting procedure comes with probabilistic guarantees and does not to require computationally expensive learning of the hidden parameters of the model in order to return an estimate of the model order.  The estimate can be further refined by running a spectral learning procedure that does learn the parameters.
\\
Our approach relies on the assumption that the parameter vectors that characterize each of the topics are randomly distributed.  We show that with high probability, the least singular value of the random matrix resulting from collecting these parameter vectors will be well-bounded.  Roughly speaking, randomly distributed topics will be unlikely to be too correlated with each other.  We show that as a result, the approximate number of latent factors can be predictably reovered from the spectral characteristics of the observable first and second moments of the data.
\\
For LDA, the requisite moments can be efficiently computed from the sufficient statistics of the model, namely the term-document co-occurrence matrix.  The usefulness of our procedure is illustrated by the following proposition for the usual case where the number of topics $K$ and the vocabulary size (or dimensionality) $V$ are such that $K = O(V), K < V$  (though we also present results for the more general case $K \le V$ in this paper):

\begin{proposition}\label{main}
Suppose we have an LDA topic model over a vocabulary of size $V$ with concentration parameter $\beta_0 \le \infty$, and we wish to determine how many nonzero topics $K$ there are in the corpus.  Suppose $K = O(V)$ and $K<V$ almost surely.  Then, for $V$ large enough, if we gather $N \ge O(\frac{\ln(V/\delta)}{\epsilon^2})$ independent samples as in Lemma \ref{alpha_concentration}, we can recover the number of topics whose expected proportion is greater than $\epsilon$, with probability greater than $1 - \delta$.
\end{proposition}

The results which allow us to prove this guarantee also provide new insights on sample complexity bounds for spectral learning of mixture models, in particular excess correlation analysis (ECA) \cite{anandkumar_hsu_kakade_12}.  These spectral algorithms have garnered attention partly because they offer better scalability to large data sets than MCMC methods, and partly because they provide probabilistic guarantees on sample complexity that are elusive for MCMC methods.  However, sample complexity results in previous literature bound the estimation error and sample complexity of learning the latent parameter matrix $\Phi$ in terms of $\Phi$ itself: given that in practice $\Phi$ is unknown beforehand, this is of limited practical utility for assessing the confidence of the estimate.  In contrast, our results allow sample complexity to be expressed directly in terms of the known quantity $V$:

\begin{proposition}\label{main2}
Suppose we have an LDA topic model over a vocabulary of size $V$.  Suppose the number of topics $K < V$ is fixed, and the variance of the entries of the latent word-topic matrix $\Phi$ is fixed and finite.  Then, for $V$ large enough, if we gather $N \ge  O(V^2)$ independent samples, we can recover the parameter matrix $\Phi$ with error less than $O(V)^{3/2}$, with probability greater than $1 - \delta$.
\end{proposition}

Taken together, these two results increase the usefulness of spectral algorithms for mixture models by allowing the number of topics to be set in a data-driven manner, and by providing more explicit sample complexity guarantees, giving the user a better idea of the quality of the learned parameters.  This brings spectral methods closer to providing a guaranteed and computationally efficient  alternative for nonparametric Bayesian models.

\subsection{Limitations of existing approaches to model order estimation}
Nonparametric Bayesian models such as the Hierarchical Dirichlet Process (HDP) \cite{teh_jordan_beal_blei_06}  have been useful in addressing the problem of model order estimation.  These models allow a distribution over an infinite number of topics.  When HDP is fitted using a Markov chain Monte Carlo (MCMC) sampling algorithm to optimize posterior likelihood, new topics are sampled as necessitated by the data.  However, training a nonparametric model using MCMC can be impractically slow for the large sample sizes likely to be encountered in many real-world applications.  As is common for MCMC methods, the Gibbs sampler for HDP is susceptible to local optima \cite{griffiths_jordan_tenenbaum_blei_04}.  Indeed, maximum likelihood estimation of topic models in general has been shown to be NP-hard \cite{arora_ge_moitra_12}.  
\\
Another class of methods is based on learning models for a finite range of topics, and then optimizing some function of likelihood or performing a likelihood-based hypothesis test over this range (e.g., the Bayes factor method, or optimization of the Bayesian Information Criterion, Akaike Information Criterion, or perplexity).  Not only do these methods suffer from the same susceptibility to local minima, but they are even more computationally intensive than nonparametric methods when the range of model orders under consideration is large.  This is because the latent parameters of the model must be learned for every single model order under consideration in order to compute the likelihoods as a basis for comparison.  The range of candidate model orders must be pre-specified by the user.  Computational complexity increases linearly as the size of the range under consideration increases.  In addition, they have been outperformed by nonparametric methods in experimental settings \cite{griffiths_jordan_tenenbaum_blei_04}.

On the other hand, spectral learning methods \cite{arora_ge_moitra_12, anandkumar_hsu_kakade_12} have been shown to provide asymptotic guarantees of exact recovery and to be computationally efficient.  However, these techniques require specifying the number of latent factors beforehand, and in some cases produce highly unstable results when the number of latent factors is underestimated \cite{kulesza_rao_singh_13}.  Therefore, a guaranteed procedure for estimating the true number of latent factors should increase the practicality of these methods for learning probabilistic models.

\subsection{Outline}
We will first provide a brief overview of the assumptions of the LDA generative model and discuss how our method is motivated by the spectral learning approach in Section \ref{setting}. In Section \ref{randmatrix}, we adapt non-asymptotic results concerning the singular values of random matrices to this setting. Practicioners interested in implementing our model order estimation method can consult Section \ref{algorithm}, where we describe our procedure for finding the number of topics, demonstrate that our method outperforms a nonparametric Bayesian method on an experimental setting taken from the literature, and discuss some other implications of our results for the accuracy of algorithms for learning $\Phi$, . 

\subsection{Notation}
For a vector $\mathbf{x}$, $\| \mathbf{x} \|$ is the Euclidean norm and $dist(\mathbf{x}, W)$ is the Euclidean distance between $\mathbf{x}$ and a subspace $W$.  For a matrix $A$, $A^{+}:= (A^TA)^{-1}A^T$ is the Moore-Penrose pseudoinverse; $\sigma_i(A)$ is the $i^{th}$ largest singular value, $\lambda_i(A)$ is the largest eigenvalue; and $\|A\|=\sigma_1(A)$ is the operator norm. \emph{a.s.} is "almost surely," and $w.p.$ is "with probability."

\section{Background} 
\subsection{Latent Dirichlet Allocation}
\label{setting}
Latent Dirichlet Allocation \cite{blei_ng_jordan_03} is a generative mixture model widely used in topic modeling.  This model assumes that the data comprise a \emph{corpus} of \emph{documents}.  In turn, each document is made up of discrete, observed \emph{word tokens}.  The observed word tokens are assumed to be generated from $K$ latent topics as follows:

\begin{algorithmic}
\FOR{ each topic $k$}
\STATE Choose a distribution $\boldsymbol{\phi}^{(i)}$ over words from a Dirichlet distribution $\boldsymbol{\phi}^{(i)} \sim$ Dir$(\boldsymbol{\beta})$ .
\ENDFOR
\STATE Collect these vectors into a matrix $\Phi = [\boldsymbol{\phi}^{(1)} | ... | \boldsymbol{\phi}^{(K)}]$  where each topic distribution vector is a column of $\Phi$.  \\
\FOR{ each document $d$ in the corpus}
\STATE Choose a distribution $\mathbf{h}_d$ over the topics, from a Dirichlet distribution $\mathbf{h}_d \sim$ Dir$(\boldsymbol{\alpha})$.\\
\FOR { each word token $v$ in $d$}
\STATE Choose topic $z_v$ from the document's distribution over topics $z_v \sim$ Mult$(\mathbf{h}_d)$. \\
\STATE Choose a word type from the topic's distribution over words $w_v \sim $ Mult$(\boldsymbol{\phi}^{(z_v)})$. For $w_v = i$ represent the word token by $\mathbf{x}_v:= \mathbf{e}_i$ (the $i^{th}$ canonical basis vector). 
\ENDFOR
\ENDFOR
\end{algorithmic}

In the generative process above, the concentration parameter $\beta_0 := \sum^V_{i=1}{\beta_i}$ can be seen as controlling how fine-grained the topics are; the smaller the value of $\beta_0$, the more distinguishable the topics are from each other.   The relative magnitude of each $\alpha_i$ represents the expected proportion of word tokens in the corpus assigned to topic $i$.  The concentration parameter $\alpha_0:= \sum^K_{k=1} \alpha_k$ governs how topically distinct documents are (in the limit $\alpha_0 \rightarrow 0$, we have a model where each document has a single topic rather than a mix of topics \citep{two_svds_suffice}).

\subsection{Spectral properties of mixture models}
\label{spectral}
For a large class of mixture models including LDA and Gaussian Mixture Models, the observed data can be represented as a sequence of exchangeable vectors $\{\mathbf{x}, \mathbf{x}', \mathbf{x}'', ... \}$ that are conditionally independent given a latent factor vector $\mathbf{h}$ which is assumed to be strictly positive.  For instance, in an LDA model each data point (word token) can be represented as a canonical basis vector $x$ of dimensionality $V$, where $V$ is the vocabulary size (number of distinct terms).  The $i$-th elment of $\mathbf{x}$ is equal to 1 if the word token that it represents is observed to belong to class $i$, and 0 otherwise.  For LDA, $\mathbf{h}$ determines the mixture of topics present in a particular document. Therefore $h$ is a vector whose support is \emph{a.s.} equal to the number of nonzero topics (the model order).  

Although the sufficient statistics of LDA can be represented in other, more succinct ways, this representation turns out  to be more than a curiosity.  To see why, observe that under this representation the conditional expectation of the observed data generated by the models can be represented as a linear combination of some latent matrix $\Phi$ (known in LDA as the word-topic matrix) and the latent membership vector $\mathbf{h}$:
$$
\mathbf{E}[\mathbf{x} |\mathbf{h} ] = \Phi \mathbf{h}.
$$
For these mixture models, the principal learning problem is to estimate $\Phi$ efficiently and accurately.  Using the equation above and the conditional independence of any three distinct observed vectors $\mathbf{x}, \mathbf{x}', \mathbf{x}''$ given $\mathbf{h}$ in the LDA model, we can derive equations for the expectations of the moments of the observed data in terms of $\Phi$.  In particular, the expected first moment, which is the vector of the expected probability masses of the terms in the vocabulary, can be written as
\begin{equation} \label{firstmomenteq}
M_1:=\mathbf{E}[\mathbf{x}] = \Phi \mathbf{E}[\mathbf{h}],
\end{equation}
and the expected second moment, which is the matrix of the expected cross-correlations between any two terms in the vocabulary, can be written as
\begin{equation} \label{secondmomenteq}
M_2: = \mathbf{E}[\mathbf{x x}^{'T}] = \Phi \mathbf{E}[\mathbf{h h}^T]  \Phi^T, \mathbf{x} \neq \mathbf{x}'.
\end{equation}
Analogous expressions for even higher moments can be expressed using tensors. In fact, \citet{anandkumar_hsu_kakade_12} were able to develop fast spectral algorithms for learning the hidden parameters of mixture models from the second- and third-order moments of the data by taking advantage of this relationship.  The resulting algorithm, excess correlation analysis (ECA), comes with probabilistic guarantees of finding the optimal solution, unlike MCMC approaches.  In the case of LDA, the only user-specified inputs to the ECA spectral algorithm are the supposed number of topics $\bar{K}$ and the concentration parameter $\alpha_0$ governing the distribution of the membership vector $\mathbf{h}$.  The matrix $\Phi$ is treated as fixed, but unknown.  
\\
Note that Eqs. (\ref{firstmomenteq}) and (\ref{secondmomenteq}) demonstrate an explicit linear-algebraic relationship between the latent parameter matrix $\Phi$, the expected moments of the data $M_1$ and $M_2$, and the expected moments of $h$.  In fact, for LDA, $\boldsymbol{\alpha}: = \mathbf{E}[\boldsymbol{h}]$ is the vector that specifies the expected proportion of data points assigned to each topic across the entire data set-- roughly speaking, if $\frac{\alpha_i}{\sum^K_{k=1}\alpha_k} = 0.5$ we expect about half of the word tokens in the data set to belong to topic $i$.  Therefore, the model order is the number of nonzero topics (i.e., the support) of $\boldsymbol{\alpha}$.  In the case of LDA, some elementary computation (cf. \cite{two_svds_suffice} Thm. 4.3) demonstrates that $\alpha$ can be written as a product of $M_1$, $M_2$, and $\Phi$ as follows:

\begin{equation} \label{alphaeq}
\boldsymbol{\alpha} I = \alpha_0(\alpha_0 + 1)\Phi^{+}(M_2 - \frac{\alpha_0}{\alpha_0 + 1} M_1 M_1^T)\Phi^{+T},
\end{equation}

where $\Phi^{+}$ is the Moore-Penrose pseudoinverse of $\Phi$ and $\alpha_0 := \sum^K_{k=1}\alpha_k$.  This suggests that $\boldsymbol{\alpha}$ and therefore the number of nonzero topics can be recovered by first learning $\Phi$ and then estimating $\alpha$ according to Eq. \ref{alphaeq}.  The true number of topics $K$ is then equal to the number of $\alpha_k$ such that $\alpha_k >0$.  However, the number of latent factors $\bar{K}$ must be speciefied beforehand in ECA, since the algorithm involves a truncated matrix decomposition and a truncated tensor decomposition.  For low-dimensional data sets, it is possible to do this by setting $\bar{K} = V$.  However, the time complexity of ECA scales as $O(\bar{K}^5)$ and the space complexity scales as $O(\bar{K}^3)$ due to the storage and decomposition of the third moment tensor, \cite{tensor_decompositions_12}, so this approach is not tractable for even moderately-sized datasets.  On the other hand, it is not possible to determine with any certainty whether we have captured all the non-zero topics if we set $\bar{K}<V$ when $K$ is unknown.  This is because when $\bar{K} < K$, then ECA learns highly unstable estimates of $\Phi$, which results in incorrect estimates of $\alpha$.  For instance, consider the following toy example: set $\alpha = [.2, .3, .5]$.  Set 
$\Phi =
\left( \begin{smallmatrix}
0 & 0.8 & 0.4 \\
0.4 & 0.1 & 0.3 \\
0.3 & 0 & 0.1 \\
0.3 & 0.1 & 0.2 \\
\end{smallmatrix}\right)$.

If we try to recover the first two values of $\alpha$ from the moments by running ECA with $\bar{K}=2$, we get $\alpha_2 = 2.5 \times 10^{-5}$.  In a practical setting where a finite number of noisy samples are used to estimate the moments, one might conclude that $\alpha_2$ is noise and that there is only one topic in this model.  Similar parameter recovery problems arise when using low-rank approximations for learning spectral algorithms for other models (see \cite{kulesza_rao_singh_13} for some Hidden Markov Model examples).  Thus, iterative methods where $\bar{K}$ is increased or decreased until $\alpha_{\bar{K}} = 0$ for some $\bar{K}$ are uncertain to provide the correct result.

We suggest a novel approach in this paper, based on singular value bounds.  Observe that taking the singular values of both sides of Eq. \ref{alphaeq} yields:

\begin{align} \label{alphaineq}
\alpha_k &=  \sigma_k\left(\alpha_0(\alpha_0 + 1)\Phi^{+}(M_2 - \frac{\alpha_0}{\alpha_0 + 1} M_1 M_1^T)\Phi^{+T}\right)\notag\\
&\le \sigma_1(\Phi^{+})^2 \sigma_k(M_2 - \frac{\alpha_0}{\alpha_0 + 1} M_1 M_1^T)\notag\\
&\le \sigma_K(\Phi)^2 \sigma_k(M_2 - \frac{\alpha_0}{\alpha_0 + 1} M_1 M_1^T).
\end{align}

Thus, rather than learning $\Phi$, we need only find some reasonably sharp bound on the least singular value of $\Phi$.  If we treat the matrix $\Phi$ as a random matrix (as in standard Bayesian approaches to LDA) and place an approximate bound on the variance of the entries of $\Phi$, then $\Phi$ has very predictable spectral characteristics for reasonably large $V$.  To prove this, we must adapt some recent results from random matrix theory.  In random matrix theory, finding the least singular values of random matrices  is often referred to as resolving the so-called "hard edge" of the spectrum.  While most work on the hard edge of the spectrum has focused on settings where the matrices are square and all entries are i.i.d. with mean zero, these conditions do not hold in the case of $\Phi$ for Dirichlet mixture models such as LDA.  We use some elementary facts about Dirichlet random variables to adapt the known results to the matrices of interest in our setting. 


Note that $M_1$ and $M_2$ are not precisely known either, but it is relatively straightforward to derive estimators for them from the observed data.  These estimators can be proven to be reasonably accurate via application of standard tail bounds for the eigenvalues and singular values of random matrices.

Thus, we can show that the observed moments of the data contain enough information to reveal the number of underlying topics to arbitrary accuracy with high probability, given enough samples.  The principles behind our results can be extended to any exchangeable mixture models that can be represented as in Eqs. (\ref{firstmomenteq}) and (\ref{secondmomenteq}), though we will work with the LDA model to make our analysis concrete.

\subsection{Assumptions}
We place some further conditions on the LDA model that allow well-behaved spectral properties.  These conditions are generally equivalent to those for ECA \cite{two_svds_suffice}, with the exception of our assumptions on $\beta_0$:
\begin{itemize}
\item The matrix $\Phi$ is of full rank.  Note  that this condition follows \emph{a.s.} from the generative process described above  \cite{chafai_10}.
\item The concentration parameters $\alpha_0$ and $\beta_0$ are approximately known.  Intuitively, as $\beta_0$ increases, the topics are less distinguishable from each other.  Note that varying this assumption only affects our model by increasing the number of samples required to learn the number of topics within a certain level of accuracy.  For simplicity of presentation, our derivations below assume that the entries $\beta_i=  \beta_0/V$ for all $i = 1, ..., V$.  This is known as a \emph{symmetric} Dirichlet prior and is equivalent to a uniform distribution on the simplex \cite{bordenave_caputo_chafai_12}.  Setting a symmetric prior on $\beta$ is standard procedure in most applications of Dirichlet mixture models; for an empirical justification of this practice, see \cite{wallach_mimno_mccallum_09}.
\item In the worst case, the number of topics is equal to the size of the vocabulary, and $K = O(V)$ \emph{a.s.}.  In most applications of Dirichlet topic models, the number of topics is in the tens or hundreds, and the size of the vocabulary is in the hundreds or thousands.
\end{itemize}

Under the assumptions and generative model above, we attempt to recover the number of topics within a margin of error defined by the expected probability mass of the topics, as follows:
\begin{definition}
A topic is $\epsilon$-relevant iff the expected proportion of data points in the corpus belonging to the topic exceeds $\epsilon$.  That is, a topic is $\epsilon$-relevant iff $\frac{\alpha_i}{\alpha_0} \le \epsilon$.
\end{definition}

Our procedure, as described below, is guaranteed to find at least all $\epsilon$-relevant topics with low probability of detecting topics to which no words are assigned in the corpus.  As long as $\beta_0<\infty$, $\epsilon$ converges to 0 as the number of samples increases.  For a fixed number of samples and a fixed failure probability $\delta$, the relevance threshhold for recovered topics $\epsilon$ increases when we wish to recover less distinguishable topics (i.e., as $\beta_0$ increases).

\section{Singular Value Bounds} \label{randmatrix}
In this section we provide tail bounds on the smallest singular values of rectangular Dirichlet random matrices.  Similar results can be derived for other Markov random matrices.  These bounds closely mirror the work of \cite{tao_vu_08}, \cite{tao_vu_09}, and \cite{rudelson_vershynin_09}, and depend on probabilistic bounds on the distance between any given random vector corresponding to a column of a random matrix and the subspace spanned by the vectors corresponding to the rest of the columns. The estimation of these distances is much simplified for random vectors with independent entries, but for a Dirichlet random matrix, the entries in each column are dependent, as they must sum to one.  Fortunately Dirichlet random vectors are related to vectors with independent entries in an elementary way.
\begin{fact}\label{gamma}
Define a vector $\boldsymbol{\gamma}^{\theta} \in \mathbf{R}^K$ such that $\gamma^{\theta}_{i} \sim$ Gamma$(\beta_0/V, \theta)$ for some $\beta_0, \theta >0$ for all $i = 1, ..., V$.  Then the scaled vector $\boldsymbol{\phi} = \frac{\boldsymbol{\gamma}^{\theta}}{\sum^V_{i = 1}{\gamma^{\theta}_i}} \sim$ Dir$(\beta_0/V).$
\end{fact}
\begin{corollary} \label{mintheta}  For any Dirichlet random matrix $\Phi$ with i.i.d. columns, and for the corresponding Gamma random matrix $\Gamma^{\theta}$ with indpendent entries $\Gamma^{\theta}_{ij} \sim$ Gamma$(\beta_0/V, \theta)$, we have that
$$
\sigma_{K}(\Phi) \ge {\frac{\sigma_{K}(\Gamma^{\theta})}{\max_j{(\sum^{V}_{i=1}{\Gamma^{\theta}_{ij}}})}}.
$$
\end{corollary}

\begin{proof} See \cite{bordenave_caputo_chafai_12} Section 2 and Lemma B.4.



\end{proof}

\subsection{Singular value bounds for matrices with i.i.d. entries}
The following singular value bound for square matrices follows from \citet{tao_vu_09} Corollary 4:
\begin{thm}
\label{sqmatrixbound}
Suppose $\Gamma$ is an $V \times V$ random matrix with independent, identically distributed entries with variance 1, mean $\mu < \infty$, and bounded fourth moment.  For any $\delta > 0$ there exist positive positive constants $c_1, c_2$ that depend only on $\mu$ such that
$$
\mathbf{P}(\sigma_{K}(\Gamma) \le  c_1 /V^{1+c_2}) \le  \delta V^{-c_1}
$$
\end{thm}

Though this bound applies also to rectangular matrices (i.e., cases where the number of topics grows more slowly than $V$) by the Cauchy Interlacing Theorem of singular values (cf. \cite{horn_johnson_90}), this bound is not sharp when $K \ll V$.  The following result follows from adapting the arguments in \citet{tao_vu_krishnapur_10} Section 8:

\begin{thm}\label{vershynin}
Let $V > K$ be positive integers.  Suppose $\Gamma$ is an $V \times K$ Gamma random matrix with independent, identically distributed entries with variance 1, mean $\mu < \infty$.  Then for every $\delta >0$ there exist $c_1$ and $V_0$ that depend only on the moments of $\Gamma$ such that, for all $V \ge V_0$.

$$
\mathbf{P}(\sigma_{K}(\Gamma) \le  c_1 \sqrt{\frac{V -K}{K}}) \le \delta'
$$
\end{thm}
In order to prove Theorem \ref{vershynin}, we need two results, presented here without proof:
\begin{proposition} \label{distancetail}
(Distance Tail Bound; \cite{rudelson_vershynin_09} Thm. 4.1).  Let $\Gamma_j$ be a vector in $\mathbf{R}^V$ whose coordinates are independent and identically distributed random variables with unit variance and bounded fourth moment.  Let $W$ be a random subspace in $\mathbf{R}^V$ spanned by $K$ vectors,  whose coordinates are independent and identically distributed random variables with bounded fourth moment and unit variance, independent of $\Gamma_j$. Then for every $\tilde{c}_1 >0$, we have $C,c, \tilde{c}_0$ that depend only on the moments such that
$$
\mathbf{P}(dist(X, W) < \tilde{c}_1 \sqrt{V - K}) \le (C \tilde{c}_1)^{V-K} + \exp(-cV).
$$
\end{proposition}
\begin{remark}Although this result is stated for mean zero random variables, see discussion in e.g., \citet{tao_vu_krishnapur_10}, Prop. 5.1 for discussion as to why it can be extended to noncentered random variables.
\end{remark}
\begin{lemma}\label{negsecnd}
(Negative Second Moment; \cite{tao_vu_krishnapur_10} Lemma A.4).  Let $1 \le K \le V$ and let $\Gamma$ be a full rank $V \times K$ matrix with columns $\Gamma_1, ..., \Gamma_K \in \mathbf{R}^V$.  For each $1 \le i \le K$, let $W_i$ be the hyperplane generated by the $K-1$ remaining columns of $\Gamma$.  Then
$$
\sum^{K}_{j = 1}{\sigma_j(\Gamma)^{-2}} = \sum^{K}_{j = 1}{dist(\Gamma_{j}, W_j)}^{-2}
$$
\end{lemma}
Now we prove Theorem \ref{vershynin}.
\begin{proof}
Proof of Theorem \ref{vershynin}.   Let $\Gamma$ be a $V \times K$ random matrix as above.
By Lemma \ref{negsecnd}, we have that
\begin{equation}
\label{sumsigmas}
{\sigma_1(\Gamma)}^{-2} + ... + {{\sigma}_{K}(\Gamma)}^{-2} = \sum^{K}_{j = 1}{dist(\Gamma_{j}, W_j)}^{-2}.
\end{equation}
By Proposition \ref{distancetail} and the union bound, \emph{w.p.} $1 -  (CV\tilde{c}_1)^{V-K} + \exp(-cV)$ we have $dist(\Gamma_{j}, W_j) \ge \tilde{c}_1 \sqrt{V - K}$ for all $j$.  Thus, with this probability, the right-hand side of Eq. (\ref{sumsigmas}) is less than $\frac{K}{\tilde{c}_1^2 (V - K)}$.  On the other hand, as the $\sigma_j(\Gamma)$ are ordered decreasingly, the left-hand side of Eq. (\ref{sumsigmas}) is at least ${\sigma_{K}(\Gamma)}^{-2}.$
It follows that, \emph{w.p.} $1 - (CV\tilde{c}_1)^{V-K} + \exp(-cV)$,
$$
\sigma_{K}(\Gamma) \ge  c_1 \sqrt{\frac{ V - K}{ K}} 
$$
thus completing the proof.
\end{proof}

\subsection{Singular value bounds for Dirichlet random matrices}
Now we are ready to derive a singular value bound for Dirichlet random matrices.
\begin{thm} \label{dirichletbound}
Let $\Phi$ be a random $V \times K$ matrix whose columns are independent identically distributed random vectors drawn from a symmetric Dirichlet distribution with parameter vector with concentration parameter $\beta_0$.  Then for every $\delta' >0$ there exist $c_1$ and $V$ that depend only on the moments of $\Phi$ such that, for all $V \ge V_0$.:

$$
 \mathbf{P}\left(\sigma_{K}(\Phi ) \le c_0 \sqrt{\frac{V -K }{\beta_0 VK}}\right) \le \delta'.
$$ 

\end{thm}

\begin{proof}
 Recall that for a symmetric Dirichlet distribution with concentration parameter $\beta_0$, each entry of the $V$-dimensional vector drawn from this distribution has mean $\beta_0/V$.  Fix $\bar{\theta}:= \sqrt{V/\beta_0}$.  Observe that $\Gamma^{\bar{\theta}}_{ij}$ has variance 1 and mean $\sqrt{\beta_0/V} < \infty$ for all $i,j$. 
Therefore, by corollary \ref{mintheta}, it follows that for any $c>0$,
\begin{align*}
\mathbf{P}(\sigma_{K}(\Phi) \le c )\le \mathbf{P}({\frac{\sigma_{K}(\Gamma^{\bar{\theta}})}{\max_j{(\sum^{V}_{i=1}{\Gamma^{\bar{\theta}}_{ij}}})}} \le c).
\end{align*}

We can exploit elementary tail bounds to control the sum in the denominator on the right-hand side above using standard concentration-of-measure results.  For instance, by Chebyshev's inequality and the mutual independence of the $K$ columns of $\Gamma^{\bar{\theta}}$, for any $u > 0$,

\begin{equation}
\label{chernoffsum}
\mathbf{P}\left(\max_j(\sum_{i=1}^{V} \Gamma^{\theta}_{ij}) \le (u + 1)\sqrt{\beta_0 V} \right) \ge (1- \frac{1}{1 + u^2 \beta_0^2})^K.
\end{equation}

By the union bound and the application of Equation \ref{chernoffsum} and Theorem \ref{vershynin}, this implies that

\begin{align*}
\mathbf{P}&\left(\sigma_{K}(\Phi) \le  {c_1(V - K )}{(u + 1)\sqrt{\beta_0 VK}}\right)
\\
&\le \mathbf{P}\left(\sigma_{K}(\Gamma^{\bar{\theta}}) \le c_1\frac{V - K }{\sqrt{V + K}}\right) 
\\
&\ \ + \mathbf{P}\left(\max_j\sum^V_{i=1}{\Gamma^{\bar{\theta}}_{i,j}} \ge (u + 1)\sqrt{\beta_0 V}\right)
\\
&\le \delta + (1 - (1 - \frac{1}{u^2\beta_0})^K).
\end{align*}
We can make the second term on the right-hand side arbitrarily small by increasing $u$, and for a fixed $u$ we can make the first term on the right-hand side arbitrarily small by decreasing $c_1$ for $V$ large enough.  Therefore, we can find a $c_0>0$ for any $\delta' >0$ such that for all $V$ large enough,
\begin{align*}
\mathbf{P}\left(\sigma_{K}(\Phi) \le c_0\sqrt{\frac{V - K}{\beta_0 VK }}\right) \le \delta'.
\end{align*}
\end{proof}

From the theorem above, we can deduce that, with high probability,

\begin{align}
\| \Phi^{+}\| = 1/\sigma_{K}(\Phi) &\le c \frac{\sqrt{ \beta_0 V K }}{V - K}   \text{, when $K \ll V$}
\\
&\le c V \sqrt{\beta_0} \text{, when $K \approx V$}.
\end{align}

\subsection{Sample concentration lemmas}
We are able to bound the error in estimating $\alpha$ from a sample thanks to sample concentration lemmas for singular values that are analogoous to more well-known concentration lemmas for scalar random variables (e.g., Markov's inequality).



\begin{lemma}\label{alpha_concentration}
\end{lemma}
(TO BE FINISHED)

\section{Applications and Experiments} \label{algorithm}
\subsection{Topic number estimation}
Although we are unable to  compute the estimator $\hat{\boldsymbol{\alpha}} = \Phi^{+}(\hat{M}_2 - \frac{\alpha_0}{\alpha_0 + 1}\hat{M}_1 \hat{M}_1^T)\Phi^{+T}$ without knowledge of $\Phi$, we can use Theorem \ref{dirichletbound} to provide an upper bound for the elements $\hat{\boldsymbol{\alpha}}$.

Define $\tilde{\alpha}_k:= c \beta_0 \frac{ V  K }{(V - K)} \sigma_k(\hat{M}_2 - \frac{\alpha_0}{\alpha_0 + 1}\hat{M}_1 \hat{M}_1^T)$.  We can now apply Theorem \ref{dirichletbound} to Eq \ref{alphaineq} to infer that there is a constant $c$ such that, for $V$ large enough,
\begin{align}
\label{alpha_ineq}
\hat{\alpha}_k &\le \sigma_K(\Phi^{+})^2 \sigma_k(\hat{M}_2 - \frac{\alpha_0}{\alpha_0 + 1}\hat{M}_1 \hat{M}_1^T)  \notag
\\
&\le \tilde{\alpha}_k
\end{align}
with high probability $1 - \delta'$ that depends on $c$ (the constant $c$ can be chosen arbitrarily so that the probability $\delta'$ is negligible \footnote{For most applications, we recommend $c \approx 1$.  We computed the least singular value for $10^7$ randomly generated Dirichlet random matrices with $\beta_0 \in (0.1V, 10V)$ $K/V \in [0.1, 0.9]$ and $V \in \{1000\}$; all of these matrices were dominated by $c = 1$.} ).

This suggests the following procedure to estimate the number of topics:

\begin{algorithmic}[1]
\REQUIRE $N$, hyperparameters $\alpha_0$, $\beta_0$, error tolerance $(\epsilon, \delta)$ according to $\delta$.
\STATE Compute the term-document matrix $C$, where $C^{(\ell)}$ represents the count vector for document $\ell$.  
\STATE Compute 'plug-in' estimates of the first and second moments of the data (\cite{tensor_decompositions_12} Section 6.1):
\begin{itemize}
\item $\hat{M}_1 \leftarrow \frac{1}{D}\sum^D_{\ell = 1} \frac{C_{\ell}}{\sum^m_{i = 1} C_{\ell, i}}$
\item $\hat{M}_2  \leftarrow \frac{1}{D} \sum^D_{\ell = 1} \frac{{C_{\ell} C_{\ell}^T - diag(C_{\ell})}}{(\sum^m_{i = 1} C_{\ell, i})(\sum^m_{i = 1} C_{\ell, i} - 1)}$
\end{itemize}
\STATE $\hat{M}_{1,2} \leftarrow \hat{M}_2 - \frac{\alpha_0}{\alpha_0 + 1}\hat{M}_1 \hat{M}_1^T$.
\STATE $k \leftarrow 1$, $\tilde{\alpha}_1 \leftarrow 1$.
\WHILE{$\tilde{\alpha}_{k}/\alpha_0 >\epsilon/2 $}
\STATE $k \leftarrow k + 1$.
\STATE $\tilde{\alpha}_{k} \leftarrow \alpha_0(\alpha_0+1) c \beta_0\frac{ V k}{V - k }\sigma_{k-1}(\hat{M}_{1,2})$ 
\ENDWHILE
\STATE (Optional) Run the ECA algorithm (\cite{two_svds_suffice} Algorithm 2) with $k$ as the number of topics and estimate $\hat{\Phi}$ and compute $\hat{\alpha}$ as in \ref{alphaeq}.
\STATE  $k \leftarrow \sum_k \mathbf{1}\{\hat{{\alpha}}_k/\alpha_0>\epsilon/2\}$
\RETURN $k$ as the estimate of $K$.
\end{algorithmic}
\begin{figure}[t]
\label{algfig}
\caption{Model order estimation performance. \emph{ Left}: Our procedure. \emph{Right}:  hLDA procedure.}
\centering
\includegraphics[ width=8cm,height=3cm, trim = 0mm 0mm 0mm 0mm]{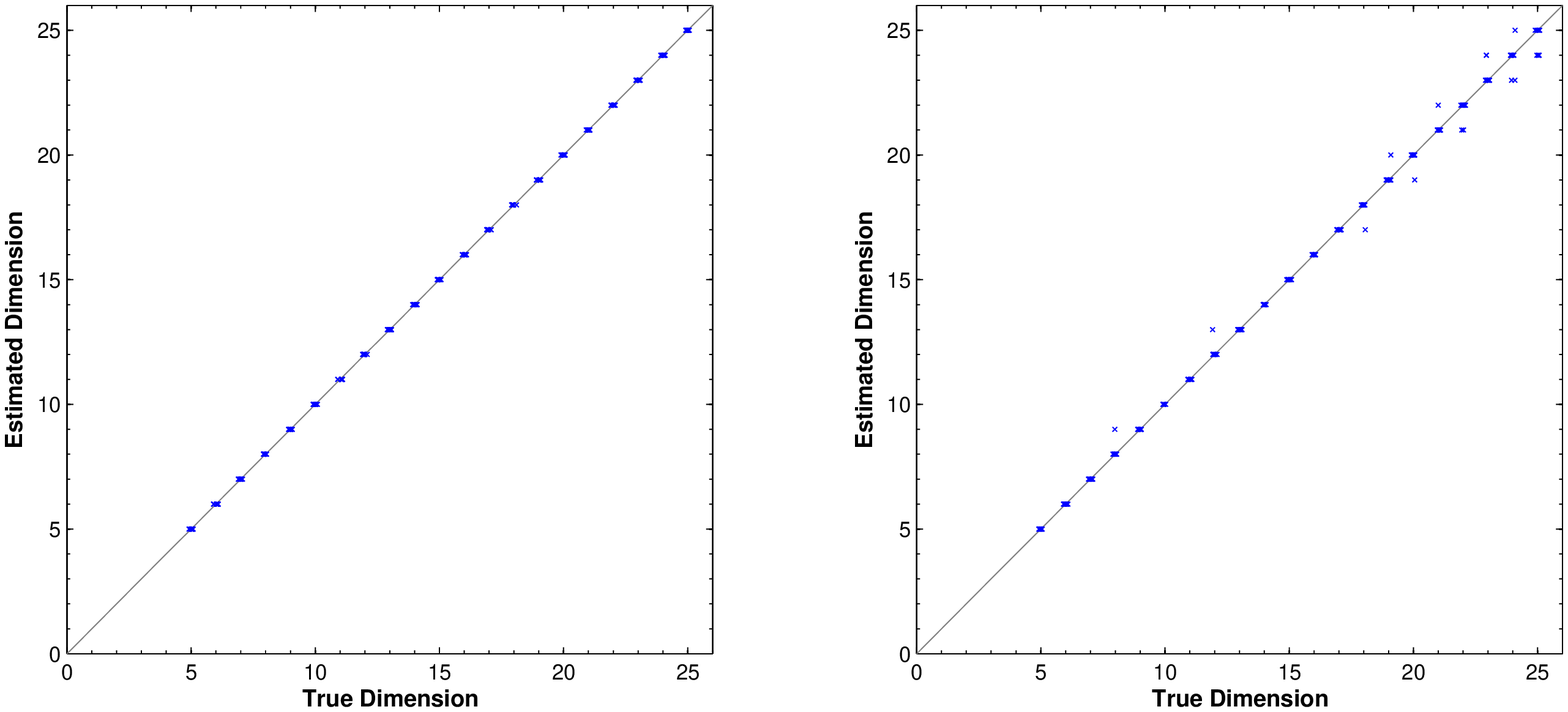} 
\end{figure}

\subsubsection{Empirical comparison to nonparametric Bayesian method}
To compare the performance of our procedure  against previous model order estimation methods, we replicated the same experimental setting used to demonstrate the model selection capabilities of hLDA \cite{griffiths_jordan_tenenbaum_blei_04}.   hLDA (a Gibbs sampling method for the nonparametric equivalent of LDA using the Chinese restaurant process prior) was shown to be much faster and accurate than the Bayes factors method (a likelihood-based hypothesis-testing method) in this setting.   210 corpora of 1000 10-word documents each were generated from an LDA model with $K \in \{5, ... , 25\}$, a vocabulary size of 100, and word-topic matrix $\Phi$ with columns randomly generated from a symmetric Dirichlet ($\beta_i = 0.1$ for $i = 1, ... ,V$, so $\beta_0 = 10$) and $\alpha_0 = 1$.  

hLDA requires the input of a concentration parameter $\gamma$ that controls how frequently a new topic is introduced, so the authors set $\gamma = 1$.  Since Gibbs sampling is subject to local maxima, so the sampler is randomly restarted 25 times for each corpus.  Each time, the sampler is burned in for 10000 iterations and subsequently samples are taken 100 iterations apart for another 1000 iterations.  The restart with the highest average likelihood over the post-burn-in period is selected, and the number of topics for this restart that had non-zero word assignments throughout the burn-in period is selected as the hLDA prediction of model order.  We used the Java implementation of the hLDA Gibbs sampler by \citet{bleier_10}.

For our spectral model selection procedure, we set our topic relevance sensitivity threshhold at $\epsilon:= 3 \times 10^{-2}$, which corresponds to an expected error rate of $\delta <1.5 \times 10^{-3}$.  We implemented our procedure using the \textsc{Matlab} standard library.
Both methods are somewhat sensitive to $\alpha_0$ and $\beta_0$, so we set these parameters to the ground truth for both methods, just as in \cite{griffiths_jordan_tenenbaum_blei_04}.


Figure \ref{algfig} shows that our model outperforms hLDA for this experimental setting (points are jittered slightly to reveal overlapping points). Our procedure correctly estimated the model order for all of the 210 corpora without the optional ECA step, whereas for hLDA the error rate was 10 out of 210.  \cite{griffiths_jordan_tenenbaum_blei_04} reported an error rate of 15 out of 210 for hLDA in this experimental setting, and an error rate of 80 out of 210 for the Bayes factors method.

The running time for hLDA Gibbs sampling procedure was $6040$ sec per corpus on a single thread of a machine with an eight-core 2.67Ghz CPU, while the running time for the spectral model selection procedure without the ECA step was 0.252 sec per corpus.  However, hLDA learns the latent matrix $\Phi$ while estimating the model order.  Including the ECA step in our spectral model selection procedure to learn $\Phi$, the running time increases to 2.05 sec per corpus.
\subsection{Convergence and learnability of spectral methods}
The learnability and sample complexity of spectral algorithms for mixture models depend crucially on the latent variable matrix $\Phi$ being well-conditioned.  For instance \cite{two_svds_suffice}'s algorithm for learning LDA comes with the following guarantee:
\begin{thm}
\label{SVDcomplexity}
(\cite{two_svds_suffice} Thm 5.1).
Fix $\delta \in (0,1)$.  Let $p_{min} = \min_i\frac{\alpha_i}{\alpha_0}$ and let $\sigma_{K}(\Phi)$ denote the smallest (non-zero) singular value of $\Phi$.  Suppose that we obtain $N \ge (\frac{(\alpha_0 + 1)(6 + 6\sqrt{\log{(3/\delta)}})}{p_{min}\sigma_{K}(\Phi)^2})^2$ independent samples of $\mathbf{x}, \mathbf{x}', \mathbf{x}''$ in the LDA model.  \emph{w.p.} greater than $1 - \delta$, the following holds: for $\theta \in \mathbf{R}^K$ sampled uniformly over the sphere $\mathcal{S}^{K-1}$, \emph{w.p.} greater than 3/4, Algorithm 5 in \cite{two_svds_suffice} returns a set $\{ \hat{\Phi}_1, ..., \hat{\Phi}_K \}$ such that there exists a permutation $\pi$ of the columns of $\Phi$ so that for all $i \in \{1, ..., K\}$
\begin{align*}
\|\Phi_i - \hat{\Phi}_{\pi(i)}\| = O\left(\frac{(\alpha_0 + 1)^2K^3}{p^2_{min} \sigma_K(\Phi)^3}\frac{1 + \sqrt{log(1/\delta)}}{\sqrt{N}}\right).
\end{align*} 
\end{thm}
Theorem \ref{dirichletbound} allows us to replace the dependence on $\Phi$ by a dependence on $V$, $K$, and $\beta_0$:
\begin{corollary}
Let $\alpha_0$, $\delta$, $p_{min}$, $\theta$, $\pi$, and $\{\hat{\Phi}_1, ... , \hat{\Phi}_K\}$ be as in \ref{SVDcomplexity}.  Suppose that we obtain $N = O\left(\left(\frac{\alpha_0 + 1}{p_{min}}\right)^2 \log(1/\delta)\left(\sqrt{\frac{\beta_0 {V K}}{V - K}}\right)^2\right)$ independent samples of $\mathbf{x}, \mathbf{x}', \mathbf{x}''$ in the LDA model.  \emph{w.p.} greater than $1-\delta$, 
\begin{multline*}
\|\Phi_i - \hat{\Phi}_{\sigma(i)} \| = \\
O\left(\left(\frac{\alpha_0 + 1}{p_{min}}\right)^2 \left(\frac{\sqrt{\beta_0 V K}}{V - K}\right)^3\frac{1 + \sqrt{log(1/\delta)}}{\sqrt{N}}\right)
\end{multline*}
\end{corollary}
Proposition \ref{main2} follows from assuming that the variance parameter $\beta_1 = \beta_0/V$ of each entry remains constant as $V$ increases (so that $\beta_0 = O(V)$), and from assuming that $K$ is fixed, so that 
\begin{align*}
\sigma_K(\Phi) &= O(\frac{\sqrt{V(V + K )\beta_1}}{V - K}) \\
&= O(\sqrt{\beta_0 V})
\end{align*}
\section{Conclusions and Further Work}

In this paper, we have derived a novel procedure for determining the number of latent topics in Latent Dirichlet Allocation.  Our experiments suggest that this procedure can outperform nonparametric Bayesian models learned using MCMC. 

Our results rely on a adapting random-matrix-theoretic results to the case of rectangular noncentered matrices, and connecting these results to the spectral properties of the moments of data generated by an LDA model.  Similar random-matrix theoretic results should be applicable to the problem of finding the number of latent factors in many other mixture models with similar conditional independence properties, and we plan to present such results in future work.  






\bibliography{masterbib}

\begin{thebibliography}{18}
\providecommand{\natexlab}[1]{#1}
\providecommand{\url}[1]{\texttt{#1}}
\expandafter\ifx\csname urlstyle\endcsname\relax
  \providecommand{\doi}[1]{doi: #1}\else
  \providecommand{\doi}{doi: \begingroup \urlstyle{rm}\Url}\fi

\bibitem[Anandkumar et~al.(2012{\natexlab{a}})Anandkumar, Foster, Hsu, Kakade,
  and Liu]{two_svds_suffice}
Anandkumar, A, Foster, DP, Hsu, D, Kakade, SM, and Liu, YK.
\newblock Two {S}{V}{D}s suffice: spectral decompositions for probabilistic
  topic models and latent dirichlet allocation.
\newblock \emph{arXiv preprint arXiv:1204.6703}, 2012{\natexlab{a}}.

\bibitem[Anandkumar et~al.(2012{\natexlab{b}})Anandkumar, Hsu, and
  Kakade]{anandkumar_hsu_kakade_12}
Anandkumar, A., Hsu, D., and Kakade, S.M.
\newblock A method of moments for mixture models and hidden {M}arkov models.
\newblock \emph{JMLR: Workshop \& Conference Proc.}, 23:\penalty0 33.1--33.34,
  2012{\natexlab{b}}.

\bibitem[Anandkumar et~al.(2012{\natexlab{c}})Anandkumar, Ge, Hsu, Kakade, and
  Telgarsky]{tensor_decompositions_12}
Anandkumar, Anima, Ge, Rong, Hsu, Daniel, Kakade, Sham~M, and Telgarsky, Matus.
\newblock Tensor decompositions for learning latent variable models.
\newblock \emph{arXiv preprint arXiv:1210.7559}, 2012{\natexlab{c}}.

\bibitem[Arora et~al.(2012)Arora, Ge, and Moitra]{arora_ge_moitra_12}
Arora, S., Ge, R., and Moitra, R.
\newblock Learning topic models -- going beyond {S}{V}{D}.
\newblock In \emph{2012 IEEE 53rd Annual Symposium on Foundations of Computer
  Science}, 2012.

\bibitem[Blei et~al.(2003)Blei, Ng, and Jordan]{blei_ng_jordan_03}
Blei, D.M., Ng, Andrew~Y, and Jordan, M.I.
\newblock Latent dirichlet allocation.
\newblock \emph{J. Machine Learning Research}, 3:\penalty0 993--1022, 2003.

\bibitem[Bleier(2010)]{bleier_10}
Bleier, A.
\newblock {J}ava {G}ibbs sampler for the {H}{D}{P}, 2010.
\newblock https://github.com/arnim/HDP.

\bibitem[Bordenave et~al.(2012)Bordenave, Caputo, and
  Chafa\"{i}]{bordenave_caputo_chafai_12}
Bordenave, C., Caputo, P., and Chafa\"{i}, D.
\newblock Circular law theorem for random {M}arkov matrices.
\newblock \emph{Prob. Theory \& Related Fields}, 152:\penalty0 751--779, 2012.

\bibitem[Chafa{\"\i}(2010)]{chafai_10}
Chafa{\"\i}, D.
\newblock The dirichlet markov ensemble.
\newblock \emph{Journal of Multivariate Analysis}, 101\penalty0 (3):\penalty0
  555--567, 2010.

\bibitem[Griffiths et~al.(2004)Griffiths, Jordan, Tenenbaum, and
  Blei]{griffiths_jordan_tenenbaum_blei_04}
Griffiths, T.L., Jordan, M.I, Tenenbaum, J.B., and Blei, D.M.
\newblock Hierarchical topic models and the nested chinese restaurant process.
\newblock \emph{Advances in neural information processing systems},
  16:\penalty0 106--114, 2004.

\bibitem[Horn \& Johnson(1990)Horn and Johnson]{horn_johnson_90}
Horn, R.A. and Johnson, C.R.
\newblock \emph{Matrix Analysis}.
\newblock Cambridge University Press, 1990.

\bibitem[Kulesza et~al.()Kulesza, Rao, and Singh]{kulesza_rao_singh_13}
Kulesza, Alex, Rao, N~Raj, and Singh, Satinder.
\newblock An exploration of low-rank spectral learning.

\bibitem[Rudelson \& Vershynin(2009)Rudelson and
  Vershynin]{rudelson_vershynin_09}
Rudelson, M. and Vershynin, R.
\newblock The smallest singular value of a random rectangular matrix.
\newblock \emph{Commun. Pure Appl. Math}, 62:\penalty0 1707--1739, 2009.

\bibitem[Tao \& Vu(2008)Tao and Vu]{tao_vu_08}
Tao, T. and Vu, V.
\newblock Random matrices: the circular law.
\newblock \emph{Comm. in Contemp. Math.}, 10.02:\penalty0 261--307, 2008.

\bibitem[Tao \& Vu(2009)Tao and Vu]{tao_vu_09}
Tao, T. and Vu, V.
\newblock Smooth analysis of the condition number and the least singular value.
\newblock pp.\  0805.3167v2, 2009.

\bibitem[Tao et~al.(2010)Tao, Vu, and Krishnapur]{tao_vu_krishnapur_10}
Tao, T., Vu, V., and Krishnapur, M.
\newblock Random matrices: Universality of {E}{S}{D}s and the circular law.
\newblock \emph{The Annals of Probability}, 38\penalty0 (5):\penalty0
  2023--2065, 2010.

\bibitem[Teh et~al.(2006)Teh, Jordan, Beal, and Blei]{teh_jordan_beal_blei_06}
Teh, Y.W., Jordan, M.I., Beal, M., and Blei, D.M.
\newblock Hierarchical {D}irichlet processes.
\newblock \emph{J. Am. Stat. Assoc.}, 101, 2006.

\bibitem[Tropp(2011)]{tropp_11}
Tropp, J.A.
\newblock User-friendly tail bounds for sums of random matrices.
\newblock \emph{Found. Comput. Math.}, X:\penalty0 X, 2011.

\bibitem[Wallach et~al.(2009)Wallach, Mimno, and
  McCallum]{wallach_mimno_mccallum_09}
Wallach, H.M., Mimno, D.M., and McCallum, A.
\newblock Rethinking {L}{D}{A}: Why priors matter.
\newblock In \emph{NIPS}, volume~22, pp.\  1973--1981, 2009.

\end{thebibliography}
\bibliographystyle{icml2014}

\appendix
\begin{lemma} \label{mcdiarmid1}
(Tropp \cite{tropp_11} Thm. 5.1 (Eigenvalue Bennett Inequality).  Consider a finite sequence $\{X_j\}$ of independent, random, self-adjoint random matrices with dimension $V$, all of which have zero mean.  Given an integer $k \le V$, define $\sigma^2_k := \lambda_k \left( \sum_j \mathbf{E}(X_j^2)\right)$.  Then, for all $t \ge 0$,
$$
\mathbf{P}\left(\lambda_1(\sum_j X_j) \ge t \right) \le V \exp \left(\frac{\sigma_k^2}{\max_j\|X_j\|^2}h(\frac{\max_j\|X_j\| t}{\sigma_k^2})\right),
$$
where the function $h(u) = (1 + u)\log(1 + u) - u$ for $u \ge 0$.
\end{lemma}

\end{document}